%% file: main.tex
\documentclass[10pt]{article} 
\usepackage[accepted]{rlc}

\usepackage{amssymb}            
\usepackage{mathtools}          
\usepackage{mathrsfs}           
\mathtoolsset{showonlyrefs}     
\usepackage{graphicx}           
\usepackage{subcaption}         
\usepackage[space]{grffile}     
\usepackage{url}                

\usepackage{hyperref}
\usepackage{url}
\usepackage[utf8]{inputenc} 
\usepackage[T1]{fontenc}    
\usepackage{url}            
\usepackage{booktabs}       
\usepackage{amsfonts}       
\usepackage{nicefrac}       
\usepackage{microtype}      
\usepackage{xcolor}         
\usepackage{graphicx}
\usepackage{float}
\usepackage{amssymb}
\usepackage{mathtools}
\usepackage{amsthm}
\usepackage{multirow}
\usepackage{subcaption}
\usepackage{wrapfig}

\usepackage{thmtools, thm-restate}
\usepackage{microtype}
\usepackage{graphicx}
\usepackage{booktabs} 
\usepackage{amsmath}

\DeclareMathOperator*{\supp}{supp}
\DeclareMathOperator*{\dooperator}{do}

\DeclareMathOperator*{\EX}{\mathbb{E}}
\usepackage{thmtools, thm-restate}
\def\Pr{\mathop{\rm Pr}\nolimits}

\theoremstyle{plain}

\newtheorem{lemma}{Lemma}

\theoremstyle{definition}
\newtheorem{definition}{Definition}

\newtheorem{example}{Example}
\theoremstyle{remark}

\newenvironment{proofsketch}{%
  \proof}{\endproof}

\title{Bad Habits:
Policy Confounding and \\ Out-of-Trajectory Generalization in RL}

\author{Miguel Suau  \\
    \small{Phaidra} \\ 
    \small{Delft University of Technology} \\
    \small{miguel.suau@phaidra.ai}
    \And
    Matthijs T. J. Spaan \\
    \small{Delft University of Technology} \\
    \small{m.t.j.spaan@tudelft.nl}
    \And
    Frans A. Oliehoek \\
    \small{Delft University of Technology} \\
    \small{f.a.oliehoek@tudelft.nl} }

\begin{document}

\maketitle

\begin{abstract}
Reinforcement learning agents tend to develop habits that are effective only under specific policies. Following an initial exploration phase where agents try out different actions, they eventually converge onto a particular policy. As this occurs, the distribution over state-action trajectories becomes narrower, leading agents to repeatedly experience the same transitions. This repetitive exposure fosters spurious correlations between certain observations and rewards. Agents may then pick up on these correlations and develop simplistic habits tailored to the specific set of trajectories dictated by their policy. The problem is that these habits may yield incorrect outcomes when agents are forced to deviate from their typical trajectories, prompted by changes in the environment. This paper presents a mathematical characterization of this phenomenon, termed policy confounding, and illustrates, through a series of examples, the circumstances under which it occurs.
\end{abstract}

\section{Introduction}
\begin{quote}
\centering
    \textit{This morning, I went to the kitchen for a coffee. When I arrived, \\ I forgot why I was there, so I got myself a coffee.---}
\end{quote}
\vspace{-3pt}
How often do you do something without paying attention to your actions? Have you ever caught yourself lost in thought while washing the dishes, making coffee, or cycling? Acting out of habit is a crucial human skill as it enables us to focus on more important matters while executing routine tasks. You can commute to work while thinking about how to persuade your boss to give you a salary raise or prepare dinner while daydreaming about your next holiday in the Alps. However, habits can also lead to unintended consequences when we fail to recognize that the context has changed. You might hop in your car and drive toward work even though it is a Sunday and you want to go to the grocery store, or you might flip the switch when leaving a room even though the lights are already off.

Surprisingly, reinforcement learning (RL) agents also struggle with this same issue. This is due to a phenomenon we term \emph{policy confounding}, which reflects how policies, as a result of influencing past and future observation variables, can inadvertently induce spurious correlations. These correlations can lead to the development of seemingly sensible but incorrect habits, such as flipping the switch upon leaving a room, without confirming whether the lights are on. The problem here is that these habits can produce incorrect results when agents are forced to deviate from their usual trajectories due to changes in the environment; a problem we refer to as \emph{out-of-trajectory} (OOT) generalization. 

 \paragraph{Contributions} 
 This paper introduces and characterizes the phenomenon of \emph{policy confounding}. To do so, we provide a mathematical framework that helps us describe the different types of state representations, and reveal how, as a result of policy confounding, the agent may learn representations based on spurious correlations that do not guarantee OOT generalization. Moreover, we include a series of clarifying examples that illustrate how this occurs. Unfortunately, we do not have a complete answer for how to prevent policy confounding. However, we suggest a few off-the-shelf solutions that may help mitigate its effects. We hope this paper will create awareness among the RL community about the risks of policy confounding and inspire further research on this topic. 




\section{Example: Frozen T-Maze}\label{sec:example}
We now provide an example to illustrate the phenomenon and motivate the need for careful analysis. 
\begin{figure}
\vspace{-7pt}
     \centering
     \begin{subfigure}[b]{0.45\textwidth}
         \centering         \includegraphics[width=0.7\textwidth]{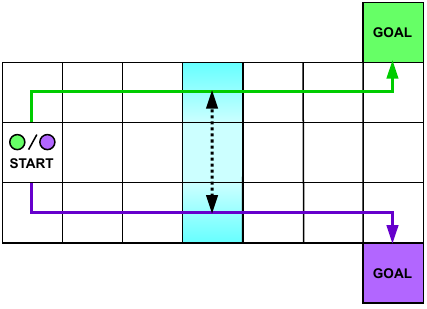}
     \end{subfigure}
     \hspace{15pt}
     \begin{subfigure}[b]{0.45\textwidth}
         \centering \includegraphics[width=0.9\textwidth]{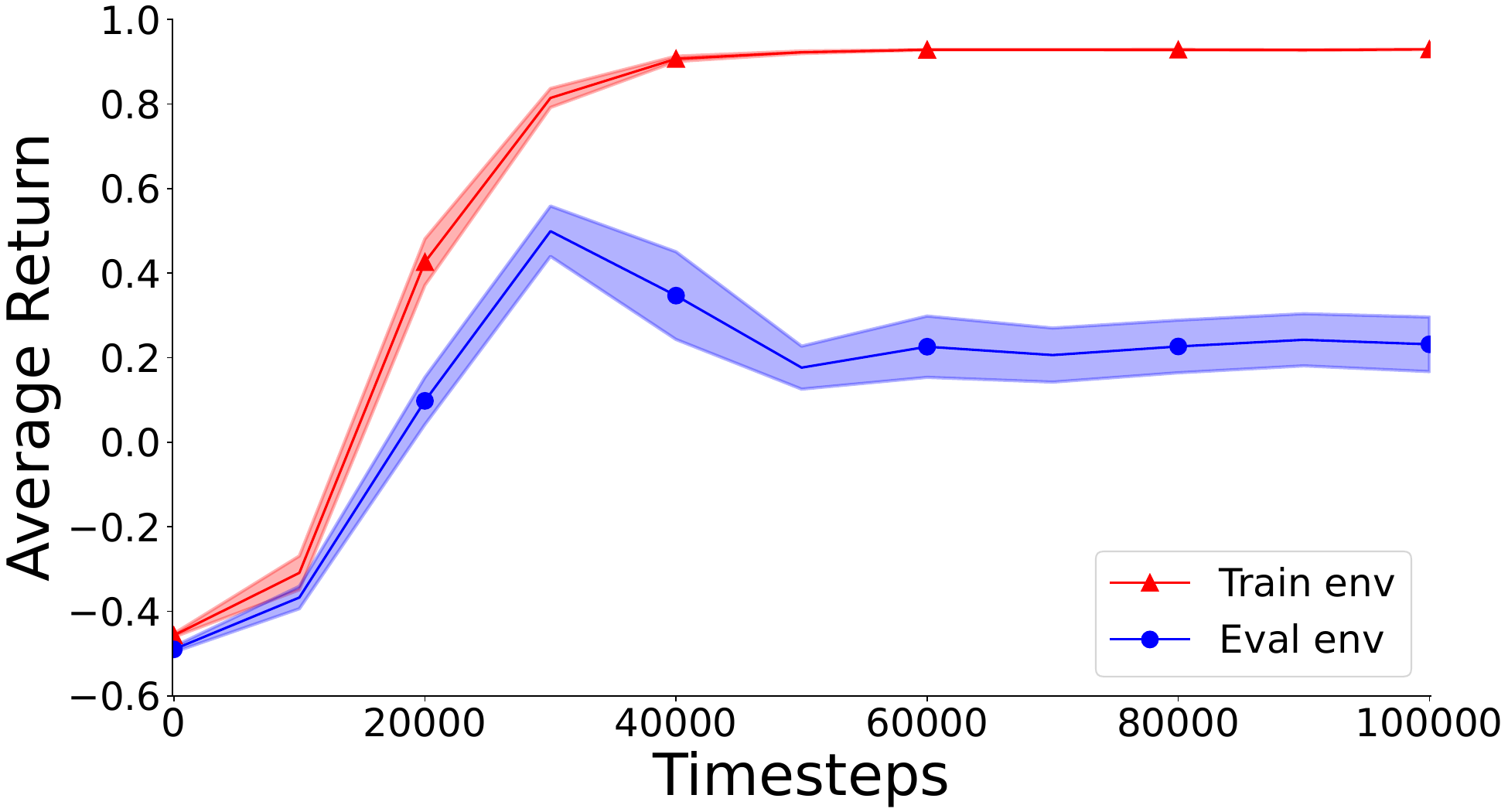}
     \end{subfigure}
     \vspace{-2pt}
      \caption{Left: An illustration of the Frozen T-Maze environment. Right: Learning curves when evaluated in the  Frozen T-Maze environment with (blue curve) and without (red curve) ice.}
      \label{fig:tmaze}
      \vspace{-10pt}
\end{figure}
Figure \ref{fig:tmaze} shows a variant of the popular T-Maze environment \citep{bakker2001reinforcement}. In this environment, the agent receives a binary signal, green or purple, at the starting location. The task for the agent is to navigate to the right and reach the correct goal at the end of the maze. The agent obtains a reward of $+1$ for moving to the green (purple) goal when having received the green (purple) signal and a penalty of $-1$ otherwise. Additionally, there is a $-0.1$ penalty per timestep to incentivize the agent to take the most direct path to the goal. At every timestep, the agent can observe its location within the maze, but the signal is only observed at the starting location. Importantly, the agent has the capability to remember past observations.

At first glance, it might seem crucial for the agent to remember the initial signal at every cell along its trajectory. After all, how else would it determine the correct goal? However, once the agent figures out the shortest path to each of the two goals (depicted by the green and purple arrows), the agent may safely forget the initial signal. The agent knows that whenever it is on the green (purple) path, it must have received the green (purple) signal. Hence, it can simply navigate toward the correct goal based solely on its current location. Sticking to this habit is optimal so long as the agent commits to these two paths. It is also essential that the environment dynamics remain the same as any changes in the agent's trajectories could erase the spurious correlation \citep{pearl2016causal} induced by the policy between the agent’s location and the correct goal.\footnote{Note that the two paths highlighted in Figure \ref{fig:tmaze} are not the only optimal paths. However, for the agent to be able to ignore the initial signal, the paths must not overlap.}

To demonstrate that this occurs in practice, we trained agents with PPO \citep{schulman2017proximal} in the original environment (train env) and evaluated them in a modified version (eval env), where there is an icy surface (depicted in blue) in the middle of the maze. The ice causes the agent to slip from the upper cell to the bottom cell and vice versa.\footnote{The ice compels the agent to take alternate trajectories by causing it to move down twice from the top cell or up twice from the bottom cell. Importantly, these trajectories are feasible within the original environment.} The plot on the right of Figure 1 shows the return averaged over 10 trials. The performance drop in the evaluation environment (blue curve) suggests that the agents' policies fail to generalize to alternative trajectories within the same environment. The ice confuses the agents, who, after being pushed away from their preferred trajectories, can no longer choose the correct goal. More details about this experiment are provided in Section \ref{sec:experiments}.


\section{Related Work}
The presence of spurious correlations in the training data is a well-studied problem in machine learning. These correlations often provide convenient shortcuts that a model can exploit to make predictions \citep{beery2018recognition}. However, a model relying on these shortcuts may experience significant performance degradation when faced with different data distributions \citep{quionero2009dataset}. \citet{di2022goal} show that RL agents may use certain environment features as proxies for choosing their actions. These features are intentionally introduced in the training environments to artificially correlate with the agent's objectives. In contrast, our work demonstrates that agents, due to policy confounding, may actively contribute to the formation of spurious correlations. 

Previous studies have reported empirical evidence of specific forms of policy confounding, revealing that in deterministic environments, agents can utilize information that correlates with their progress in an episode to determine optimal actions. This strategy is effective because, under fixed policies, features like timers \citep{Song2020Observational}, agent postures \citep{lan2023can}, or previous action sequences \citep{machado2018revisiting} can be directly mapped to the agent's state. While these studies offer various hypotheses to explain their experimental observations, we contribute an overarching theory that explains the underlying causes and mechanisms behind these results, along with a series of examples illustrating other types of policy confounding. 

Out-of-trajectory (OOT) generalization is a particular instance of the more general problem of out-of-distribution (OOD) generalization in RL \citep{kirk2023survey}. The objective of OOT generalization is not to generalize to environments with different rewards \citep{Taylor09ICML}, observations \citep{mandlekar2017adversarially,zhang2020invariant}, and transitions \citep{higgins2017darla} but simply to alternative trajectories within the same environment. In our experiments, agents are evaluated in altered environments with different dynamics. These alterations are only intended to force the agent to take alternative trajectories within the same environment. Importantly, these alternative trajectories are both possible and probable in the original environment. Example \ref{ex:OODvsOOT} illustrates the distinction between OOT and OOD.  Please refer to Appendix \ref{ap:related_work} for more details on related work.


\section{Preliminaries}\label{sec:preliminaries}

\subsection{Notation}
We denote random variables with capital letters (e.g., $S$), their corresponding values with lowercase letters (e.g., $s$), and their domains with calligraphic letters (e.g., $\mathcal{S}$). To denote the domain of a set of random variables $F =\{F^1, \ldots, F^{|F|} \}$, we use $\times \mathcal{F}$ as a shorthand for the Cartesian product $\mathcal{F}^1 \times \dots \times \mathcal{F}^{|F|}$. This notation represents all possible combinations of values for the variables in $F$.
\subsection{Problem formulation}

\begin{definition}[MDP]
 A Markov decision process (MDP) is a tuple $\langle \mathcal{S}, \mathcal{A}, T, R \rangle$, where $\mathcal{S}$ represents the set of states, $\mathcal{A}$ denotes the set of actions available to the agent,  $T : \mathcal{S} \times \mathcal{A} \to \Delta (\mathcal{S})$  is the transitions function, and $R : \mathcal{S} \times \mathcal{A} \to \mathbb{R}$ is the reward function.
\end{definition}

In particular, we focus on problems where states are represented by a set of observation variables, or factors \citep{boutilier1999decision}. This representation is common when modeling policies and value functions using function approximators \citep{sutton2018reinforcement}. These observation variables typically describe features of the agent's state in the environment.

\begin{definition}[FMDP]
A Factored Markov decision process (FMDP) is an MDP where
the set of states is defined by a set of observation variables, or factors, $F = \{F^1, \dots, F^{|F|}\}$. Each variable $F^i$ can take any of the values in its domain $\mathcal{F}^i$. Consequently, each state $s$ corresponds to a unique combination of values for the variables in $F$,  $s = \langle f^1, \dots, f^{|F|} \rangle \in \times \mathcal{F} = \mathcal{S}$.
\end{definition}

While, for simplicity, we employ the MDP formulation, the insights presented here are not exclusive to fully observable environments. In cases where the current observation variables $F$ do not satisfy the Markov property, $F$ is considered to be the history of action and observation variables, which is guaranteed to satisfy the Markov property \citep{Kaelbling98AI}.


\section{State representations}

The agent's objective is to learn a policy $\pi: \mathcal{S} \to \Delta (\mathcal{A})$ that maximizes the expected discounted sum of rewards \citep{sutton2018reinforcement}. However, learning a policy that conditions on every observation variable might be impractical, particularly in scenarios with a large number of variables. Fortunately, in many problems, not all variables are strictly essential, and compact state representations can be found that are sufficient for solving the task at hand \citep{McCallum95PhD}. This is where function approximators, such as neural networks, come into the picture \citep{franccois2018introduction, ni2024transformers}. If we use them to model policies and value functions, they will learn to ignore certain observation variables in $F$ if they are deemed unnecessary for estimating values and optimal actions.

As we shall see, the phenomenon of policy confounding plays a fundamental role in this quest for simpler state representations, tricking the function approximator into forming state representations that are based on mere spurious correlations. Before delving into these intricacies, let us establish some key definitions regarding state representations. To enhance clarity, we will use the environment introduced in Section \ref{sec:example} as a running example throughout this and the next section. 

\begin{example} \label{ex:frozen}
Refer to the first paragraph of Section \ref{sec:example} for a description of the environment. We denote the agent's location by $L$. The variable $G$ indicates whether the goal is to reach the green or purple cell. $G$ is sampled at the beginning and its value remains constant throughout the episode. The value of $G$ is hidden and only passed to the agent at the starting location through the variable $X$, representing the signal. At any other location, $X$ takes a dummy value, rendering the environment partially observable. Consequently, the set $F_t$ is the history of actions, locations, and signal variables, $F_t = \{ L_0, X_0, A_0, \dots, A_{t-1}, L_t, X_t \}$. Each unique combination of values defines a state $s_t$. Here, the subscript $t$ is used to denote that the number of variables in $F$ depends on $t$. 
\end{example}

\begin{definition}[State representation] 
A state representation is a function $\Phi: \mathcal{S} \to \bar{\mathcal{S}}$, where $\mathcal{S} = \times \mathcal{F}$, $\bar{\mathcal{S}} = \times \bar{\mathcal{F}}$, and $\bar{F} \subseteq F$. 
\label{def:state_representation}
\end{definition}

Intuitively, a state representation $\Phi(s_t)$ is a state-specific projection of a state $s \in \mathcal{S} = \times \mathcal{F}$ onto a lower-dimensional space $\bar{S} = \times \bar{\mathcal{F}}$ defined by a subset of its variables, $\bar{F} \subseteq F$. We use $\{s\}^\Phi =  \{s' \in \mathcal{S}: \Phi(s') = \Phi(s)\}$ to denote the equivalence class of $s$ under $\Phi$. In Example \ref{ex:frozen}, a potential state representation could be $\Phi(s_t) = \langle l_0, x_0 \rangle$ for all $s_t \in \mathcal{S}$. This representation retains only $L_0$ and $X_0$, ignoring all other variables in $F$. Hence, all states that share the same values for $L_0$ and $X_0$ belong to the same equivalence class.

\subsection{Markov state representations}\label{sec:markov_representations}
Not all state representations are sufficient to learn the optimal policy; some, like the one discussed in the above paragraph, may exclude variables that carry valuable information for the task at hand.

\begin{definition}[Markov state representation] A state representation $\Phi(s_t)$ is said to be Markov if, for all $s_t, s_{t+1} \in \mathcal{S}$, $a_t \in \mathcal{A}$, 
\begin{equation*}
 R(s_t , a_t) = R(\Phi(s_t), a_t) \quad
\text{and} \quad \sum_{s'_{t+1} \in \{s_{t+1}\}^\Phi} T(s'_{t+1} \mid s_t, a_t) = \Pr(\Phi(s_{t+1}) \mid \Phi(s_t), a_t),
\end{equation*}
where $R(\Phi(s_t), a_t)$ denotes the reward $R(s'_t, a_t)$ at any $s'_t \in \{s_t\}^\Phi$.
\label{def:markov_state}
\end{definition}
The above definition is analogous to the notion of bisimulation \citep{Dean97AAAI, Givan03AIJ} or model-irrelevance state abstraction \citep{Li06ISAIM}. Representations satisfying these conditions are guaranteed to be behaviorally equivalent to the original representation. That is, for any given policy and initial state, the expected return (i.e., cumulative reward; \citealp{sutton2018reinforcement}) is the same when conditioning on the full set of observation variables $F$ or on the Markov state representation $\Phi$.

\begin{definition}[Minimal state representation] A state representation $\Phi^*:  \mathcal{S} \to \bar{\mathcal{S}}^*$ with $ \bar{\mathcal{S}}^* = \times \bar{\mathcal{F}}^*$ is said to be \emph{minimal}, if all other state representations $\Phi:  \mathcal{S}  \to \bar{\mathcal{S}}$ with $\bar{\mathcal{S}}  = \times \bar{\mathcal{F}}$ and $\bar{F} \subset \bar{F}^*$, for some $s \in \mathcal{S}$, are not Markov.
\label{def:minimal_state}
\end{definition}



In simpler terms, $\Phi^*$ is \emph{minimal} when none of the remaining variables can be removed while the representation remains Markov. Hence, we say that a minimal state representation $\Phi^*$ is a sufficient statistic of the full set of observation variables $F$. In Example \ref{ex:frozen}, representations like $\Phi(s_t) = \langle x_0, a_t, l_t \rangle$ or $\Phi(s_t) = \langle x_0, x_t, l_t \rangle$ are valid Markov state representations. However, only representations that exclusively retain the initial signal $X_0$ and the agent's current location $L_t$ (i.e., $\Phi^*(s_t) = \langle x_0, l_t \rangle$) are considered minimal, as only these two variables are necessary to capture rewards and transitions.

\begin{definition}[Superfluous variable] Let $\{\bar{F}^*\}_{\cup \Phi^*}$ be the union of variables in all possible minimal state representations. 
A variable $F^i \in F$ is said to be superfluous, if $F^i \notin \{\bar{F}^*\}_{\cup \Phi^*}$.
\label{def:superfluous}

In Example \ref{ex:frozen}, any variable other than the signal and the current location, $F \setminus \{X_0, L_t\}$, is superfluous.

\end{definition}

\subsection{$\pi$-Markov state representations}
Considering that the agent's policy will rarely visit all states, the notion of Markov state representation might be overly strict. We now introduce a relaxed version that guarantees the representation is Markov when following specific policy $\pi$.

\begin{definition}[$\pi$-Markov state representation] A state representation $\Phi^\pi(h_t)$ is said to be $\pi$-Markov if, for all $s_t, s_{t+1} \in \mathcal{S}^\pi$, $a_t \in \supp(\pi( \cdot \mid s_t))$,  
\begin{equation*}
    R(s_t , a_t) = R^\pi(\Phi^\pi(s_t), a_t)
    \quad \text{and} \quad \sum_{s'_{t+1} \in \{s_{t+1}\}^\Phi_\pi} T(s'_{t+1} \mid s_t, a_t) = \Pr^\pi(\Phi^\pi(s_{t+1}) \mid \Phi^\pi(s_t), a_t),
\end{equation*}
where $S^\pi \subseteq S$ denotes the set of states visited under $\pi$, $R^\pi(\Phi^\pi(s_t), a_t)$ is the reward $R(s'_t, a_t)$ at any $s'_t \in \{s_t\}^\Phi_\pi$, with $\{s\}^\Phi_\pi =\{s' \in S^\pi: \Phi^\pi(s') = \Phi^\pi(s)\} $, and $\Pr^\pi$ is probability under $\pi$.
\label{def:pi_markov_state}
\end{definition}

\begin{definition}[$\pi$-minimal state representation] \label{def:pi_minimal_state}
A state representation $\Phi^{\pi*}:  \mathcal{S}^\pi \to \bar{\mathcal{S}}^{\pi*}$ with $\bar{\mathcal{S}}^{\pi*} = \times \bar{\mathcal{F}}^{\pi*}$ is said to be \emph{$\pi$-minimal}, if all other state representations $\Phi:  \mathcal{S}^\pi \to \bar{\mathcal{S}}^\pi$ with $\bar{\mathcal{S}}^\pi = \times \bar{\mathcal{F}}$ and $\bar{F} \subset \bar{F}^{\pi*}$, for some $s \in \mathcal{S}^\pi$, are not $\pi$-Markov.
\end{definition}

The next result demonstrates that a $\pi$-Markov state representation $\Phi^\pi$ requires at most the same variables, and in some cases fewer, than a minimal state representation $\Phi^*$, while still satisfying the Markov conditions for those states visited under $\pi$, $s \in S^\pi$. 

\begin{restatable}{proposition}{propositionone}
Let $\mathbf{\Phi^*}$ be the set of all possible minimal state representations, where every $\Phi^* \in \mathbf{\Phi^*}$ is defined as $\Phi^*: \mathcal{S} \to \bar{\mathcal{S}}^*$ with $\bar{\mathcal{S}}^* = \times \bar{\mathcal{F}}^*$. For all $\pi$ and all $\Phi^* \in \mathbf{\Phi^*}$, there exists a $\pi$-Markov state representation $\Phi^{\pi}: \mathcal{S}^{\pi} \to \bar{\mathcal{S}}^{\pi}$ with $\bar{\mathcal{S}}^{\pi} = \times \bar{\mathcal{F}}^{\pi}$ such that for all $s \in \mathcal{S}^\pi$, $\bar{F}^\pi \subseteq \bar{F}^*$. Moreover, there exist cases where $\bar{F}^\pi$ is a proper subset, $\bar{F}^\pi \subset \bar{F}^*$.
\label{prop:subset}
\end{restatable}


It is clear that $\Phi^\pi$ will never require more variables than the corresponding minimal state representation $\Phi^*$ because, as per Definition \ref{def:markov_state}, $\Phi^*$ captures all the essential information. The situation where $\bar{F}^\pi \subset \bar{F}^*$ arises with particular policies that exclusively visit a subset of states. In such cases, the agent may require fewer variables within that subset to accurately capture rewards and transitions. Take, for instance, a policy that makes the agent stay put. The $\pi$-minimal representation under such a policy is the empty set, $\Phi(s_t) = \emptyset$ for all $s_t \in \mathcal{S}^\pi$, as the agent consistently receives the same reward and does not move from the initial state. 

\section{Policy Confounding}

Judging from the previous example, it might be tempting to assume that having $\bar{F}^\pi \subset \bar{F}^*$ is merely an incidental outcome of following a policy $\pi$ that visits a subset of states, where some variables coincidentally happen to be unnecessary. Moreover, considering that $\bar{F}^*$ is constructed to capture the essential variables necessary for the task, one may further conclude that a policy $\pi$ inducing representations such that $\bar{F}^\pi \subset \bar{F}^*$ can never be optimal. However, as demonstrated by the following example, it turns out that the states visited by a particular policy, especially if it is the optimal policy, tend to contain a lot of redundant information. This is particularly true in environments where future states are heavily influenced by past actions.

\begin{figure}
\vspace{-12pt}
\centering
\includegraphics[width=\textwidth]{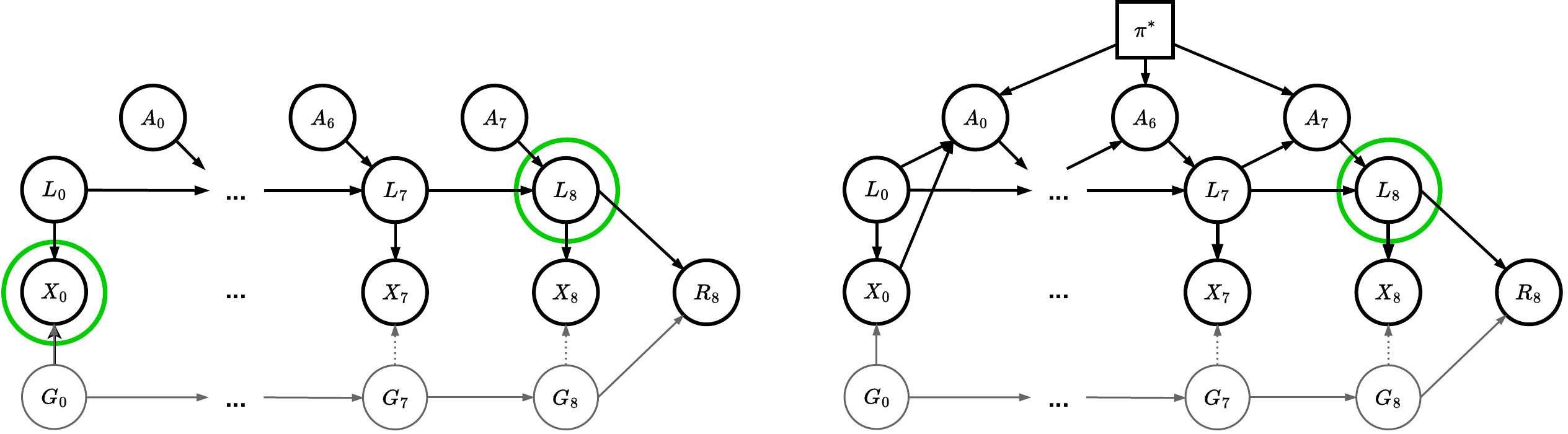}
\caption{Two DBNs representing the dynamics of the Frozen T-Maze environment, when actions are sampled at random (left), and when they are determined by the optimal policy (right). The green circles highlight the $\pi$-mininal state representation in each of the two cases.} 
\label{fig:tmaze_dbn}
\vspace{-10pt}
\end{figure}

Let us revisit Example \ref{ex:frozen}. Figure \ref{fig:tmaze_dbn} shows two dynamic Bayesian networks (DBNs) describing the environment's dynamics: one with random action sampling (left) and the other with actions determined by the optimal policy (right). Both networks are unrolled from $t=0$ to $t=8$, with intermediate nodes omitted for simplicity. Suppose we aim to predict the reward $R_8$ given $s_8 = \langle l_0, x_0, a_0, \dots, a_7, l_8, x_8 \rangle$. In the case of random action sampling (left DBN), to predict $R_8$, one needs $X_0$ and $L_8$. This is because $\langle A_0, \dots, A_7 \rangle$ appear as exogenous and can take any possible value. Hence, the reward could be either $-0.1$ (per timestep penalty), $-1$ (wrong goal), or $+1$ (correct goal) depending on the actual values of $X_0$ and $L_8$. As established in Section \ref{sec:markov_representations}, we know that $\Phi^*(s_t) = \langle x_0, l_t \rangle$ is a minimal state representation.

Conversely, when actions are determined by the optimal policy $\pi^*$ (right DBN), knowing $L_8$ alone suffices to determine $R_8$. The reason is that, under $\pi^*$, the agent always takes the action 'move up' at the starting location when receiving the green signal or 'move down' when receiving the purple signal and then follows the shortest path toward each of the goals. As shown by the diagram, this makes the action $A_0$, and thus all future agent locations, dependent on the initial signal $X_0$. Hence, the agent’s location $L_t$ becomes a proxy for $X_0$, allowing the agent to ignore $X_0$ and still predict transitions and rewards. Consequently, from $t=1$ onward, $\Phi^{\pi^*}(s_t) = l_t$ is a $\pi$-minimal state representation (Definition \ref{def:pi_minimal_state}) as it constitutes a sufficient statistic of the state $s_t$ under $\pi^*$. For the same reason, from $t=1$, actions may also condition only on $L_t$.

\begin{wrapfigure}{r}{0.36\textwidth}
\vspace{-20pt}
\centering
\includegraphics[width=0.33\textwidth]{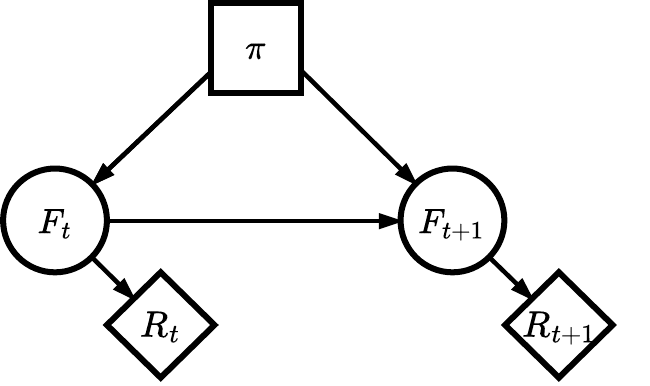}
\caption{A DBN illustrating the phenomenon of policy confounding. The policy opens a backdoor path that can affect conditional relations between the variables in $F_t$ and $
F_{t+1}$.} 
\label{fig:policy_confounding}
\vspace{-15pt}
\end{wrapfigure}

The phenomenon highlighted by the previous example results from a spurious correlation induced by the optimal policy between the initial signal $X_0$ and the agent's future locations $\langle L_1, \dots, L_8\rangle$. Generally speaking, this occurs because policies act as confounders, opening backdoor paths between future observation variables $F_{t+1}$ and the variables in the current state $F_t$ \citep{pearl2016causal}. This is illustrated by the DBN depicted in Figure \ref{fig:policy_confounding}, where the policy influences both the variables in $F_t$ and the variables in $F_{t+1}$, potentially altering their correlations.

\begin{definition}[Policy Confounding] A state representation $\Phi : \mathcal{S} \to \bar{\mathcal{S}}$ is said to be confounded by a policy $\pi$ if, for some $s_t, s_{t+1}  \in \mathcal{S}$, $a_t \in \mathcal{A}$,
    \begin{equation*}
    R^\pi(\Phi(s_t), a_t) \neq R^\pi(\dooperator(\Phi(s_t)), a_t)
    \quad \text{or} \quad
        \Pr^\pi(\Phi(s_{t+1}) \mid \Phi(s_t), a_t) \neq \Pr^\pi(\Phi(s_{t+1}) \mid \dooperator(\Phi(s_t)), a_t).
   \end{equation*}
\normalsize
\end{definition}

The operator $\dooperator(\cdot)$ is known as the do-operator, and it is used to represent physical interventions in a system \citep{pearl2016causal}. These interventions are meant to distinguish cause-effect relations from mere statistical associations. In our case, $\dooperator(\Phi(s_t))$ means setting the variables forming the state representation $\Phi(s_t)$ to a particular value and considering all possible states in the equivalence class, $s'_t \in \{s_t\}^\Phi$, (i.e., all states that share the same value for the observation variables that are `selected' by $\Phi$; independently of whether these are visited by the policy being followed). For instance, in the above example, $R^{\pi^*}(L_8 = \text{`green goal'}) = +1$ when following $\pi^*$ since we know that $X_0 = \text{`green'}$, while $R^{\pi^*}(\dooperator(L_8 = \text{`green goal'})) = \pm1$ since $X_0$ can be either `green' or `purple'.

It is easy to show that the underlying reason why a $\pi$-Markov state representation may require fewer variables than the minimal state representation is indeed policy confounding.

\begin{restatable}{theorem}{policyconfounding}
Let $\Phi^*:  \mathcal{S} \to \bar{\mathcal{S}}^*$  with $\bar{\mathcal{S}}^* = \times \bar{\mathcal{F}}^*$ be a minimal state representation. If, for some $\pi$, there is a $\pi$-Markov state representation $\Phi^{\pi}:  \mathcal{S}^{\pi} \to \bar{\mathcal{S}}^{\pi}$ with $\bar{\mathcal{S}}^{\pi} = \times \bar{\mathcal{F}}^\pi$, such that $\bar{F}^\pi \subset \bar{F}^*$ for some $s \in \mathcal{S}$, then $\Phi^\pi$ is confounded by policy $\pi$.
\end{restatable}



Finally, it is worth noting that even though, in Example \ref{ex:frozen}, the variables included in the $\pi$-minimal state representation are a subset of the variables in the minimal state representation, $\bar{F}^{\pi*} \subset \bar{F}^*$, this is not always the case, as $\bar{F}^{\pi*}$ may contain superfluous variables (Definition \ref{def:superfluous}). An example illustrating this situation is provided in Appendix \ref{ap:watch_example} (Example \ref{ex:watch_time}).

\begin{restatable}{proposition}{propositiontwo}
Let $\{\bar{F}^*\}_{\cup \Phi^*}$ be the union of variables in all possible minimal state representations. There exist cases where, for some $\pi$, there is a $\pi$-minimal state representation $\Phi^{\pi *}: \mathcal{S}^\pi \to \bar{\mathcal{S}}^{\pi *}$ with $\bar{\mathcal{S}}^{\pi *} = \times \bar{\mathcal{F}}^{\pi*}$ such that $\bar{F}^{\pi *} \setminus \{\bar{F}^*\}_{\cup \Phi^*} \neq \emptyset$.

\end{restatable}

\subsection{Why should we care about policy confounding?}\label{sec:OOT}

Leveraging spurious correlations to develop simple habits can be advantageous when resources such as memory, computing power, or data are limited. Agents can exclude variables from the state representation if they are redundant under their policy. However, the challenge is that some of these variables may be crucial to ensure that the agent behaves correctly when the context changes. In the Frozen T-Maze example from Section \ref{sec:example}, we observed how the agent could no longer find the correct goal when the ice pushed it away from the optimal trajectory. This is a specific case of a well-researched issue known as out-of-distribution (OOD) generalization \citep{quionero2009dataset, arjovsky2021out}. We refer to it as \emph{out-of-trajectory} (OOT) generalization to highlight that the problem here is that the agent is unable to generalize to alternative trajectories within the same environment. This is in contrast to previous works \citep{kirk2023survey} that address generalization to environments that differ in some way from the training environment. Example \ref{ex:OODvsOOT} illustrates the distinction between OOT and OOD.

Ideally, the agent should aim to learn representations that enable it to predict future rewards and transitions even when experiencing slight variations in its trajectory. Based on Definition \ref{def:markov_state}, we know that, in general, only a Markov state representation satisfies these requirements. However, computing such representations is typically intractable \citep{Ferns2006Methods}, and thus most standard RL methods usually learn representations by maximizing an objective function that depends on the distribution of trajectories $P^b(\tau)$ visited under a behavior policy $b$ (e.g., expected return, $\EX_{\tau \sim P^b(\tau)} \left[G(\tau)\right]$; \citealp{sutton2018reinforcement}). The problem is that $b$ may favor certain trajectories over others, which may lead to the exploitation of spurious correlations in the learned representation.

\subsection{When should we worry about OOT generalization in practice?}\label{sec:when}



\paragraph{Function approximation}
Function approximation has enabled traditional RL methods to scale to high-dimensional problems, where storing values in lookup tables is infeasible \citep{franccois2018introduction}. Using parametric functions (e.g., neural networks) to model policies and value functions, agents can learn abstractions by grouping together states if these yield the same transitions and rewards. As mentioned before,  abstractions occur naturally when states are represented by a set of variables since the functions simply need to ignore some of these variables.  However, this also implies that value functions and policies are exposed to spurious correlations. If a particular variable becomes irrelevant due to policy confounding, the function may learn to ignore it and remove it from its representation (Example \ref{ex:frozen}). This is in contrast to tabular representations, where, every state takes a separate entry, and even though there exist algorithms that perform state abstractions in tabular settings \citep{Andre02AAAI, Givan03AIJ}, these abstractions are normally formed offline before learning the policy, hence avoiding the risk of policy confounding.
%

\paragraph{Narrow trajectory distributions}
In practice, agents are less prone to policy confounding when the trajectory distribution $P^b(\tau)$ is broad (i.e., when $b$ encompasses a wide set of trajectories) than when it is narrow. This is because the spurious correlations present in certain trajectories are less likely to have an effect on the learned representations. On-policy methods (e.g., SARSA, Actor-Critic; \citealp{sutton2018reinforcement}) are particularly troublesome for this reason since the same policy being updated must also be used to collect the samples.
 Yet, even when the trajectory distribution is narrow, there is no reason why the agent should pick up on spurious correlations while its policy is still being updated. Only when the agent commits to a particular policy should we start worrying about policy confounding. At this point, lots of the same trajectories are being used for training, and the agent may \emph{`forget'} \citep{french1999catastrophic} that, even though certain variables may  no longer be needed to represent the states, they were important under previous policies. This generally occurs at the end of training when the agent has converged to a particular policy. However, if policy confounding occurs earlier during training, it may prevent the agent from further improving its policy  (\citealp{nikishin2022primacy}; please refer to Appendix \ref{ap:related_work} for more details).




\subsection{What can we do to improve OOT generalization?}\label{sec:what_do}
As mentioned in the introduction, we do not have a complete answer to the problem of policy confounding. Yet, here we offer a few off-the-shelf solutions that, while perhaps limited in scope, can help mitigate the problem in some situations. These solutions revolve around the idea of broadening the distribution of trajectories to dilute the spurious correlations introduced by certain policies.

\paragraph{Off-policy methods}
We already explained in Section \ref{sec:when} that on-policy methods are particularly prone to policy confounding since they are restricted to using samples coming from the same policy. A rather obvious solution is to instead use off-policy methods, which allow using data generated from previous policies. Because the samples belong to a mixture of policies it is less likely that the model will pick up the spurious correlations present on specific trajectories. However, as we shall see in the experiments, this alternative works only when replay buffers are large enough.
This is because standard replay buffers are implemented as queues, and hence the first experiences coming in are the first being removed. This implies that a replay buffer that is too small will contain samples coming from few and very similar policies. Since there is a limit on how large replay buffers are allowed to be, future research could explore other, more sophisticated, ways of deciding what samples to store and which ones to remove \citep{shcaul2016prioritized}.
\paragraph{Exploration and domain randomization}
When allowed, exploration may mitigate the effects of policy confounding and prevent agents from overfitting their preferred trajectories. Exploration strategies have already been used for the purpose of generalization; to guarantee robustness to perturbations in the environment dynamics \citep{eysenbach2022maximum}, or to boost generalization to unseen environments \citep{jiang2022uncertaintydriven}.  The goal for us is to remove, to the extent possible, the spurious correlations introduced by the current policy. Unfortunately, though, exploration is not always without cost. Safety-critical applications require the agent to stay within certain boundaries \citep{altman1999constrained,garcia2015comprehensive}.
When training on a simulator, an alternative to exploration is domain randomization \citep{tobin2017domain, peng2018sim, machado2018revisiting}. 
The empirical results reported in the next section suggest that agents become less susceptible to policy confounding when adding enough stochasticity. Yet, there is a limit on how much noise can be added to the environment or the policy without altering the optimal policy (\citealp[Example 6.6: Cliff Walking]{sutton2018reinforcement}).

\section{Experiments}\label{sec:experiments}

The experiments aim to (1) validate the occurrence of policy confounding as described by the theory, (2) identify the conditions under which agents are most susceptible to the effects of policy confounding and fail to generalize, and (3) assess the effectiveness of strategies proposed in the previous section in mitigating these effects. To achieve this, a series of simple grid-world environments has been devised as pedagogical examples to highlight the issue and clarify the theory. We would like to emphasize that our primary contribution lies in characterizing the phenomenon of policy confounding. The extent to which this phenomenon manifests in more realistic settings is beyond the scope of this paper. However, we believe that the failure of standard RL methods in these simplistic environments raises important concerns. Moreover, we refer the reader to Appendix \ref{ap:related_work} for a review of prior works reporting evidence of particular forms of policy confounding in high-dimensional environments.

\subsection{Experimental setup}
Agents are trained with an off-policy method, DQN \citep{Mnih15Nature} and an on-policy method, PPO \citep{schulman2017proximal}. We represent policies and value functions as feedforward neural networks and use a stack of past observations as input in the environments that require memory. We report the mean return as a function of the number of training steps. Training is interleaved with periodic evaluations on the original environments and variants thereof used for validation. The results are averaged over 10 random seeds. Refer to Appendix \ref{ap:experimental_details} for more details about the setup.
\subsection{Environments}
We run our experiments on three grid-world environments: the \textbf{Frozen T-Maze} from Section \ref{sec:example}, and the below described \textbf{Key2Door}, and \textbf{Diversion} environments.
\begin{example}\textbf{Key2Door.}
Here, the agent needs to collect a key placed at the beginning of the corridor in Figure \ref{fig:key2door_offtrack} (left) and then open the door at the end. The current observation variables do not show whether the key has already been collected. The states are thus given by the history of past locations $s_t = \langle l_0, \dots, l_t \rangle$. This is because to solve the task in the minimum number of steps, the agent must remember that it already got the key when going to the door. Yet, since during training, the agent always starts the episode at the first cell from the left, when moving toward the door, the agent can forget about the key (i.e., ignore past locations) once it has reached the third cell. As in the Frozen T-Maze example, the agent can build the habit of using its own location to tell whether it has or has not got the key yet. This, can only occur when the agent consistently follows the optimal policy, depicted by the purple arrow. Otherwise, if the agent moves randomly through the corridor, it is impossible to tell whether the key has or has not been collected. In contrast, in the evaluation environment, the agent always starts at the second to last cell, this confuses the agent, which is used to already having the key by the time it reaches said cell. A DBN is provided in Appendix \ref{ap:dbns}.
\end{example}
\begin{figure}
     \centering
     \begin{subfigure}[b]{0.4\textwidth}
         \centering    \includegraphics[width=\textwidth]{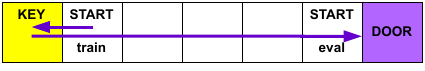}
         \vspace{0.8pt}
     \end{subfigure}
     \hspace{40pt}
     \begin{subfigure}[b]{0.4\textwidth}
         \centering
\includegraphics[width=0.95\textwidth]{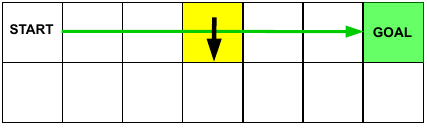}  
     \end{subfigure}
      \caption{Illustrations of the Key2Door (left) and Diversion (right) environments.}
      \label{fig:key2door_offtrack}
      \vspace{-10pt}
\end{figure}
\begin{example}
\textbf{Diversion.} Here, the agent must move from the start state to the goal state in Figure~\ref{fig:key2door_offtrack} (right). The observations are length-$8$ binary vectors. The first $7$ elements indicate the column where the agent is located. The last element indicates the row. This environment aims to show that policy confounding can occur not only when the environment is partially observable, as was the case in the previous examples, but also in fully observable scenarios. After the agent learns the optimal trajectory depicted by the green arrow, it can disregard the last element in the observation vector. This is because, if the agent does not deviate, the bottom row is never visited. Rather than forgetting past information, the agent ignores the last element in the current observation vector for being irrelevant when following the optimal trajectory.
We train the agent in the original environment and evaluate it in a version with a yellow diversion sign in the middle of the maze that forces the agent to move to the bottom row. 
A DBN is provided in Appendix \ref{ap:dbns}.
\end{example}

\subsection{Results}
\begin{figure}
\vspace{-15pt}
\centering
  \includegraphics[width=0.8\linewidth]{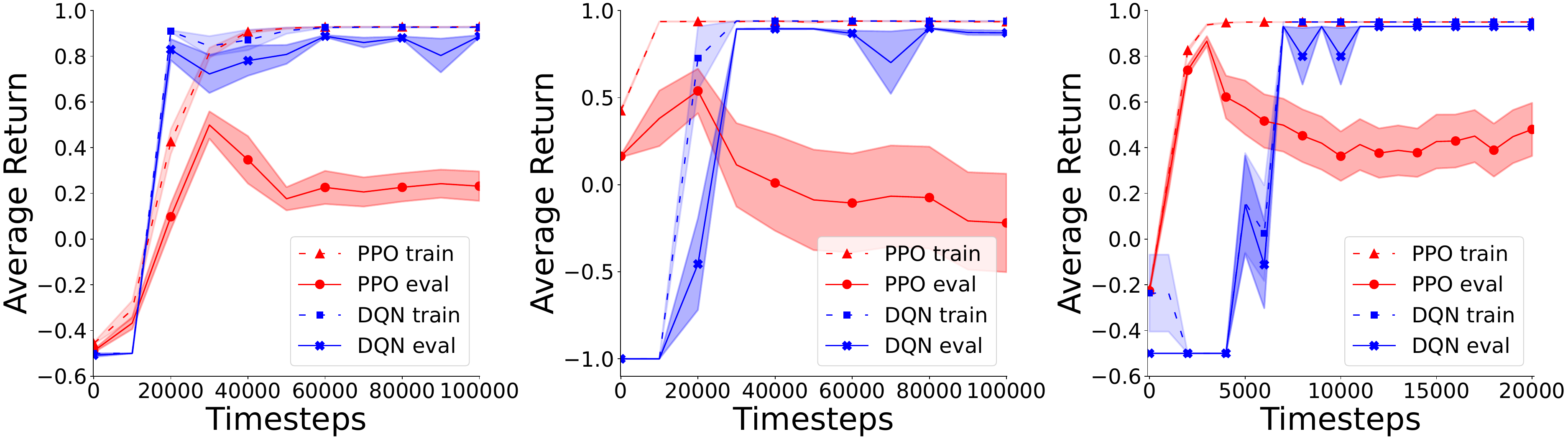}
  \captionof{figure}{DQN vs. PPO in the  train and evaluation variants of Frozen T-Maze (left), Key2Door (middle), and Diversion (right).}
  \label{fig:DQNvsPPO}
\end{figure}
\begin{figure}
  \centering
  \includegraphics[width=0.6\linewidth]{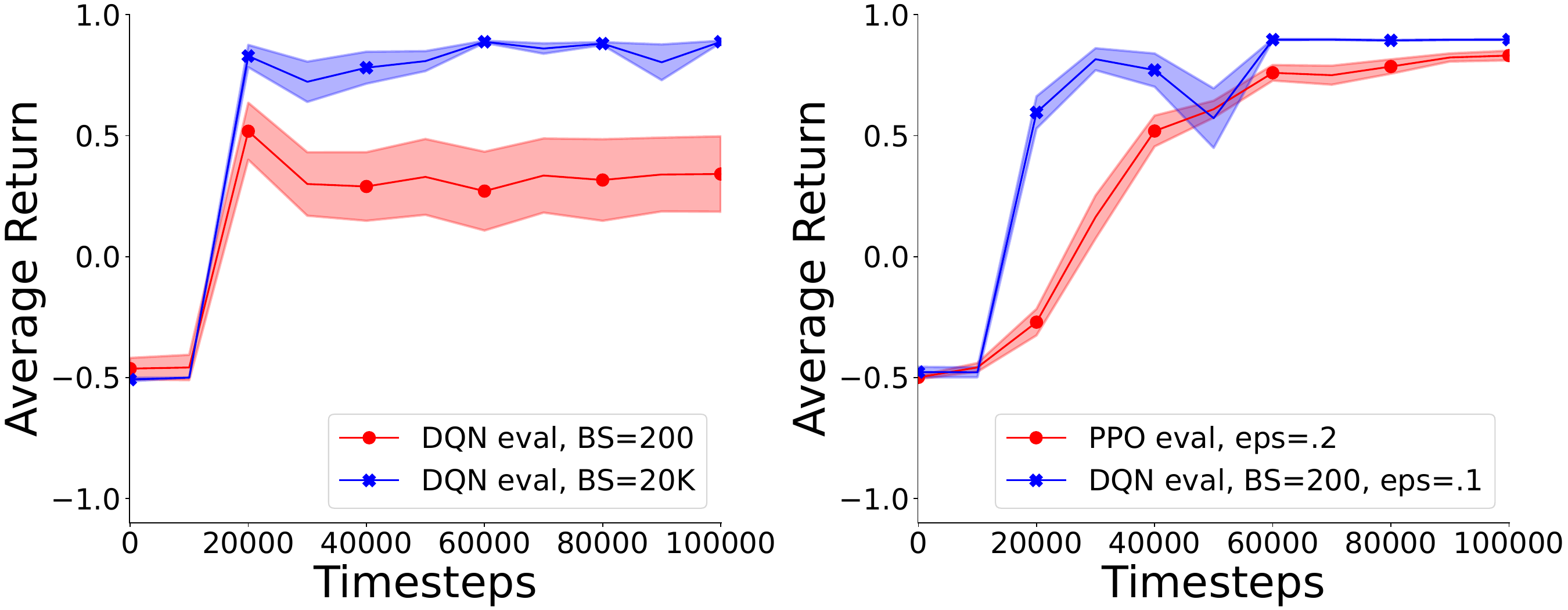}
  \captionof{figure}{Frozen T-Maze. Left: DQN small vs. large buffer sizes. Right: PPO and DQN when adding stochasticity.}
  \label{fig:exploration_buffer}
\vspace{-10pt}
\end{figure}
\paragraph{On-policy vs. off-policy}

The results in Figure \ref{fig:DQNvsPPO} consistently reveal a common trend across all three environments. PPO struggles with generalization beyond the agent's preferred trajectories. After an initial phase where the average returns on the training and evaluation environments increase (`PPO train' and `PPO eval'), the return on the evaluation environments (`PPO eval') starts decreasing when the agent commits to a particular trajectory, as a result of policy confounding. 
In contrast, since the training samples come from a mixture of policies, DQN performs optimally in both variants of the environment (`DQN train' and `DQN eval') long after converging to the optimal policy.
A visualization of the state representations learned with PPO, showing that the policy does ignore necessary variables, is provided in Appendix \ref{ap:learned_representations}.
\paragraph{Large vs. small replay buffers}
We mentioned in Section \ref{sec:what_do} that the effectiveness of off-policy methods against policy confounding depends on the size of the replay buffer. The results in Figure \ref{fig:exploration_buffer} (left) confirm this claim. The plot shows the performance of DQN in the Frozen T-Maze environment when the size of the replay buffer contains $100$K experiences and when it only contains the last $10$K experiences. We see that in the second case, the agents performance in the evaluation environment decreases (red curve left plot). This is because, after the initial exploration phase, the distribution of trajectories becomes too narrow, and the spurious correlations induced by the latest policies dominate the replay buffer. Similar results for the other two environments are provided in Appendix \ref{ap:buffer_exploration}.


\paragraph{Exploration and domain randomization}
The last experiment shows that if sufficient exploration is allowed, DQN may still generalize to different trajectories, even when using small replay buffers (blue curve right plot on Figure \ref{fig:exploration_buffer}). In the original configuration, the exploration rate $\epsilon$ for DQN starts at $\epsilon = 1$ and decays to $\epsilon = 0.0$ after $20$K steps. For this experiment, we set the final rate $\epsilon = 0.1$. 
In contrast, since exploration in PPO is controlled by the entropy bonus, which makes it hard to ensure fixed exploration rates, we add noise to the environment instead. The red curve in Figure \ref{fig:exploration_buffer} (right) shows that when the agent's actions are overridden by a random action with $20\%$ probability, the performance of PPO in the evaluation environment does not degrade after the agent has converged to the optimal policy. This suggests that the added noise prevents spurious correlations from dominating training batches. However,  it may also happen that random noise is insufficient to remove the spurious correlations, as occurs in the Key2Door environment (Figure \ref{fig:exploration_buffer_keydoor}; Appendix \ref{ap:buffer_exploration}). Similar results for Diversion are provided in Appendix \ref{ap:buffer_exploration}.

\vspace{-5pt}
\section{Conclusion}
This paper described the phenomenon of policy confounding. We demonstrated both theoretically and empirically how, as a result of following certain trajectories, agents may pick up on spurious correlations and develop habits that are not robust to trajectory deviations. We also identified the circumstances under which policy confounding is most likely to occur in practice and suggested a few ad hoc solutions that may mitigate its effects. We view this paper as a stepping stone to exploring more sophisticated solutions. An interesting avenue for future research is the integration of tools from the field of causal inference \citep{hernan2010causal, peters2017elements} to assist the agent in forming state representations grounded in causal relationships rather than mere statistical associations \citep{lu2018deconfounding, zhang2020invariant, sontakke2021causal, saengkyongam2023invariant}.

\section*{Acknowledgements}
\begin{wrapfigure}{r}{0.28\linewidth}
    \vspace{-18pt}
    \hspace{5pt}
    \includegraphics[width=0.88\linewidth]{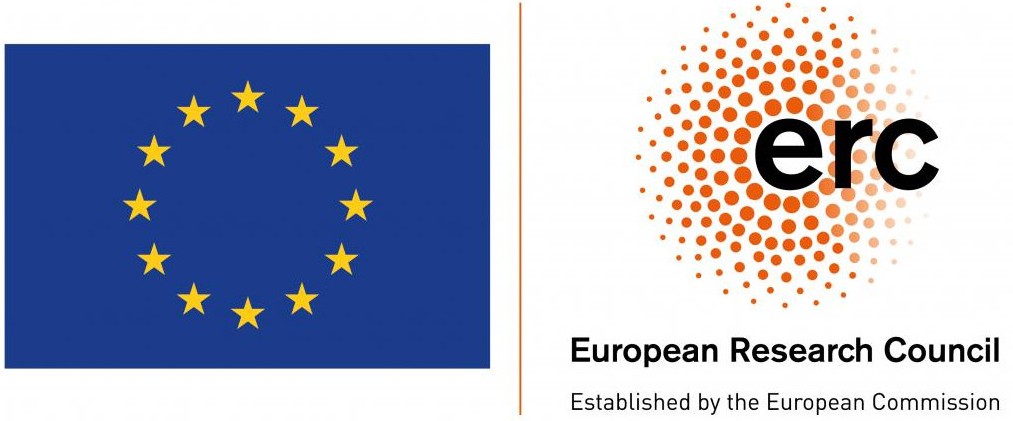}
    \vspace{-20pt}
\end{wrapfigure}
This project received funding from the European Research Council (ERC) 
under the European Union's Horizon 2020  research 
and innovation program (grant agreement No.~758824 \textemdash INFLUENCE).



\bibliography{bibliography}
\bibliographystyle{rlc}

\newpage

\appendix
\onecolumn

\section{Proofs}\label{ap:proofs}

\begin{lemma}
  Let $\mathbf{\Phi^{\pi_1*}}$ be the set of all possible $\pi$-minimal state representations under $\pi_1$, where every $\Phi^{\pi_1*} \in \mathbf{\Phi^{\pi_1*}}$ is defined as $\Phi^{\pi_1*}:  \mathcal{S}^{\pi_1} \to \bar{\mathcal{S}}^{\pi_1*}$ and $\bar{\mathcal{S}}^{\pi_1*} = \times \bar{\mathcal{F}}^{\pi_1 *}$ , and let   
 $\pi_2$ be a second policy such that for all $s_t \in \mathcal{S}^{\pi_1} \cap \mathcal{S}^{\pi_2}$, 
  $$\supp \left(\pi_2(\cdot \mid s_t)\right) \subseteq \supp \left(\pi_1(\cdot \mid s_t)\right).$$ For all $\Phi^{\pi_1*} \in \mathbf{\Phi^{\pi_1*}}$, there exists a $\pi$-Markov state representation under policy $\pi_2$, $\Phi^{\pi_2}: \mathcal{S}^{\pi_2} \to \bar{\mathcal{S}}^{\pi_2}$ with $\bar{\mathcal{S}}^{\pi_2} = \times \bar{\mathcal{F}}^{\pi_2}$, such that $\bar{F}^{\pi_2} \subseteq \bar{F}^{\pi_1*}$ for all $s_t \in \mathcal{S}^{\pi_1} \cap \mathcal{S}^{\pi_2}$. Moreover, there exist cases where $\bar{F}^{\pi_2}_t \neq \bar{F}^{\pi_1*}_t$.
  \label{lemma:one}
 \end{lemma}
\begin{proof}
First, it is easy to show that 
$$ \forall s_t \in \mathcal{S}, \supp \left(\pi_2( \cdot \mid s_t)\right) \subseteq \supp \left(\pi_1( \cdot \mid s_t)\right) \iff  \mathcal{S}^{\pi_2} \subseteq \mathcal{S}^{\pi_1},$$
and
$$ \forall s_t \in \mathcal{S}, \supp \left(\pi_2(\cdot \mid s_t)\right) = \supp \left(\pi_1(\cdot \mid s_t)\right) \iff  \mathcal{S}^{\pi_2} = \mathcal{S}^{\pi_1}.$$

In particular, $\mathcal{S}^{\pi_2} \subset \mathcal{S}^{\pi_1}$ if there is  at least one state $s'_t \in \mathcal{S}^{\pi_1} \cap \mathcal{S}^{\pi_2}$ such that
$$\supp \left(\pi_2(\cdot \mid s'_t)\right) \subset \supp \left(\pi_1(\cdot \mid s'_t)\right)$$
while 
$$\supp \left(\pi_2(\cdot \mid s_t)\right) = \supp \left(\pi_1(\cdot \mid s_t)\right)$$ 
for all other $s_t \in \mathcal{S}^{\pi_1} \cap \mathcal{S}^{\pi_2}$. 

In such cases, we know that there is at least one action $a'$ for which $\pi_2(a'_t \mid s'_t) = 0$ but $\pi_1(a'_t \mid s'_t) \neq 0$. Hence, if there was a state (or group of states) that could only be reached by taking action $a'_t$ at $s'_t$, $\pi_2$ would never visit it and thus $\mathcal{S}^{\pi_2} \subset \mathcal{S}^{\pi_1}$.

Further,  if $\mathcal{S}^{\pi_2} \subset \mathcal{S}^{\pi_1}$, we know that, for every $\Phi^{\pi_1*} \in \mathbf{\Phi^{\pi_1*}}$, there must be a  $\Phi^{\pi_2*}$ that requires, at most, the same number of variables,  $\bar{F}_t^{\pi_2} \subseteq \bar{F}_t^{\pi_1 *}$
and, in some cases,  fewer, $\bar{F}_t^{\pi_1 *} \neq \bar{F}_t^{\pi_2 *}$ (e.g., Frozen T-Maze example).



\end{proof}

\propositionone*

\begin{proof}
The proof follows from Lemma \ref{lemma:one}. We know that, in general, $\mathcal{S}^\pi \subseteq \mathcal{S}$, and if $\pi(a'_t| s'_t) = 0$ for at least one pair $a'_t \in \mathcal{A}, s'_t \in \mathcal{S}$ for which there is a state (or group of states) that can only be reached by taking action $a'_t$ at $s'_t$, then $\mathcal{S}^\pi \subset \mathcal{S}$. Hence, for every $\Phi^*$ there is a $\Phi^\pi$ such that $\bar{F}^\pi \subseteq \bar{F}^*$, and in some cases, we may have $\bar{F}^\pi \subset \bar{F}^*$ (e.g., Frozen T-Maze example).

\end{proof}

\policyconfounding*
\begin{proof}
    Proof by contradiction. Let us assume that $\bar{F}^\pi \subset \bar{F}^*$, and yet there is no policy confounding. I.e., for all $s_t, s_{t+1}  \in \mathcal{S}$, $a_t \in \mathcal{A}$,
    \begin{equation}
        R^\pi(\Phi^\pi(s_t), a_t)=  R^\pi(\dooperator(\Phi^\pi(s_t)), a_t)
        \label{eq:no_confounding1}
    \end{equation}
    and
    \begin{align}
    \begin{split}
        \Pr^\pi(\Phi^\pi(s_{t+1}) \mid \Phi^\pi(s_t), a_t) &= \Pr^\pi(\Phi^\pi(s_{t+1}) \mid \dooperator(\Phi^\pi(s_t)), a_t)
    \end{split}
    \end{align}
    First, note that the do-operator implies that the equality  must hold for \emph{all} $s'_t$ in the equivalence of $s_t$ class under $\Phi^\pi$, $s'_t \in \{s_t\}^{\Phi^\pi} = \{s'_t \in \mathcal{S}: \Phi(s'_t) = \Phi(s_t)\}$, i.e., not just those $s'_t$ that are visited under $\pi$, 
    \begin{equation}
         R^\pi(\Phi^\pi(s_t), a_t) = R^\pi(\dooperator(\Phi^\pi(s_t)), a_t) =  R(s'_t, a_t) \quad \text{for all} \quad s'_t \in \{s_t\}^\Phi
    \end{equation}
    which is precisely the first condition in Definition \ref{def:markov_state},
    \begin{equation}
        R(s_t, a_t) = R(\Phi^\pi(s_t), a_t) 
         \label{eq:dooperator_ex}
    \end{equation}
    for all $s_t \in \mathcal{S}$ and $a_t \in \mathcal{A}$.
    
    Analogously, we have that,
    \begin{align}
    \begin{split}
        \Pr^\pi(\Phi^\pi(s_{t+1}) \mid \Phi^\pi(s_t), a_t) &= \Pr^\pi(\Phi^\pi(s_{t+1}) \mid \dooperator(\Phi^\pi(s_t)), a_t)\\
        &= \Pr(\Phi^\pi(s_{t+1}) \mid \Phi^\pi(s_t), a_t)
        \label{eq:no_confounding2}
    \end{split}
    \end{align}
    where the second equality reflects that the above must hold independently of $\pi$. Hence, we have that for all $s_t, s_{t+1} \in \mathcal{S}$ and $s'_t \in \{s_t\}^\Phi$,
    \begin{align}
    \begin{split}
        \Pr(\Phi^\pi(s_{t+1}) \mid \Phi^\pi(s_t), a_t)
        &=  \Pr(\Phi^\pi(s_{t+1}) \mid \Phi^\pi(s'_t), a_t),
    \end{split}
    \end{align}
    which means that, for all $s_t, s_{t+1} \in \mathcal{S}$ and $s_t \in \mathcal{A}$,
    \begin{align}
    \begin{split}
         \Pr(\Phi^\pi(s_{t+1}) \mid \Phi^\pi(s_t), a_t) &= \Pr(\Phi^\pi(s_{t+1}) \mid s_t, a_t) \\ 
         &=\sum_{s'_{t+1} \in \{s_{t+1}\}^{\Phi^\pi}} T(s'_{t+1} \mid s_t, a_t),
         \label{eq:dooperator_prob}
    \end{split}
    \end{align}
    which is the second condition in Definition \ref{def:markov_state}.
    
    Equations \eqref{eq:dooperator_ex} and \eqref{eq:dooperator_prob} reveal that if the assumption is true (i.e., $\Phi^\pi$ is not confounded by the policy), then $\Phi^\pi$ is not just $\pi$-Markov but actually strictly Markov (Definition \ref{def:markov_state}). However, we know that $\Phi^*(s_t)$ is the minimal state representation, which contradicts the above statement, since, according to Definition \ref{def:minimal_state}, there is no proper subset of $\bar{F}^*$, for all $s_t \in \mathcal{S}$, such that the representation remains Markov. Hence, $\bar{F}^\pi \subset \bar{F}^*$ implies policy confounding.
\end{proof}
\propositiontwo*
\begin{proofsketch} Consider a deterministic MDP with a deterministic policy. Imagine there exists a variable $X$ that is perfectly correlated with the episode's timestep $t$, but that is generally irrelevant to the task. The variable $X$ would constitute in itself a valid $\pi$-Markov state representation since it can be used to determine transitions and rewards so long as a deterministic policy is followed. At the same time, $X$ would not enter the minimal Markov state representation because it is useless under stochastic policies. Example \ref{ex:watch_time} below illustrates this situation.
\end{proofsketch} 

\section{Example: Watch the Time}\label{ap:watch_example}

\begin{figure}[h!]
\centering
\includegraphics[width=0.38\textwidth]{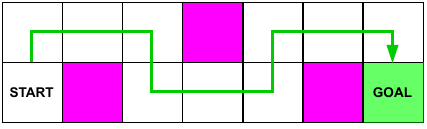}
\caption{An illustration of the watch-the-time environment.} 
\label{fig:watch_time}
\end{figure}
\begin{figure*}[h!]
\centering
\includegraphics[width=\textwidth]{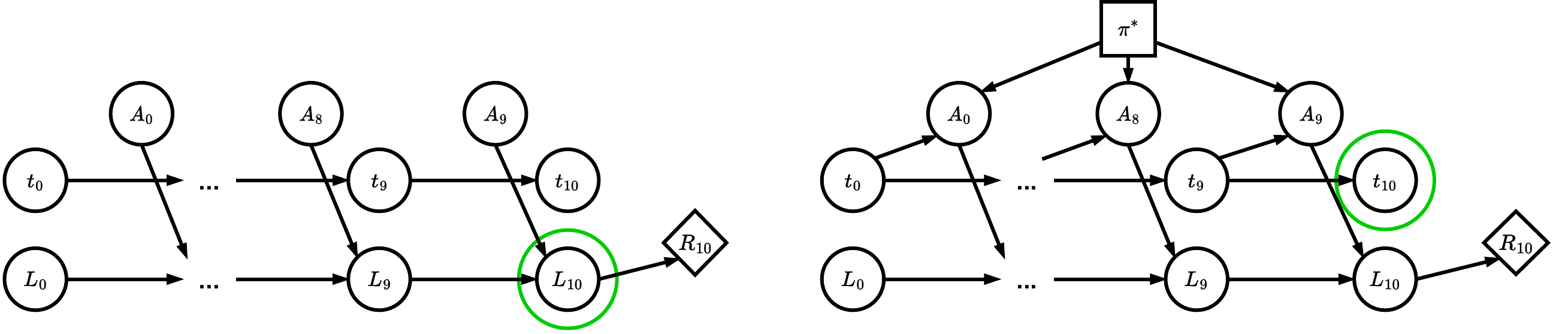}
\caption{Two DBNs representing the dynamics of the watch-the-time environment, when actions are sampled at random (left), and when they are determined by the optimal policy (right).} 
\label{fig:dbn_wathctime}
\end{figure*}
\begin{example}
\textbf{(Watch the Time)} This example is inspired by the empirical results of \citet{Song2020Observational}. Figure \ref{fig:watch_time} shows a grid world environment. The agent must go from the start cell to the goal cell. The agent must avoid the pink cells; stepping on those yields a $-0.1$ penalty. There is a $+1$ reward for reaching the goal. The agent can observe its own location within the maze $L$ and the current timestep $t$. The two diagrams in Figure \ref{fig:dbn_wathctime}
 are DBNs describing the environment dynamics. When actions are considered exogenous random variables (left diagram), the only way to estimate the reward at $t=10$ is by looking at the agent's location $L_{10}$. In contrast, when actions are determined by the policy (right diagram), $t$ becomes a proxy for the agent's location. This is because the start location and the sequence of actions are fixed. This implies that $t$ is a perfectly valid $\pi$-Markov state representation under $\pi^*$. Moreover, as shown by the DBN on the right, the optimal policy may simply rely on $t$ to determine the optimal action.
 \label{ex:watch_time}
 \end{example}
\section{Further Related Work}\label{ap:related_work}

\paragraph{Early evidence of policy confounding}
Although to the best of our knowledge, we are the first to bring forward and describe mathematically the idea of policy confounding, a few prior works have reported evidence of particular forms of policy confounding. In their review of the Arcade Learning Environment (ALE; \citealp{bellemare13arcade}), \citet{machado2018revisiting} explain that because the games are fully deterministic (i.e., initial states are fixed and transitions are deterministic), open-loop policies that memorize good action sequences can achieve high scores in ALE. Clearly, this can only occur if the policies themselves are also deterministic. In such cases, policies, acting as confounders, induce a spurious correlation between the past action sequences and the environment states. Similarly, \citet{Song2020Observational} show, by means of saliency maps, how agents may learn to use irrelevant features of the environment that happen to be correlated with the agent's progress, such as background clouds or the game timer, as clues for outputting optimal actions. In this case, the policy is again a confounder for all these, a priori irrelevant, features. \citet{zhang2018study} provide empirical results showing how large neural networks may overfit their training environments and, even when trained on a collection of procedurally generated environments, memorize the optimal action for each observation. \citet{zhang2018dissection} show how, when trained on a small subset of trajectories, agents fail to generalize to a set of test trajectories generated by the same simulator. \citet{ostrovski2021difficulty} empirically show that agents passively trained on observational data generated by other agents tend to perform poorly due to extrapolation errors caused by some of the state-action pairs being underrepresented in the data. \citet{lan2023can} report evidence of well-trained agents failing to perform well on Mujoco environments when starting from trajectories (states) that are out of the distribution induced by the agent's policy. We conceive this as a simple form of policy confounding. Since the Mujoco environments are also deterministic, agents following a fixed policy can memorize the best actions to take for each state instantiation, potentially relying on superfluous features. Hence, they can overfit to unnatural postures that would not occur under different policies. Finally, \citet{nikishin2022primacy} describe a phenomenon named `primacy bias', which prevents agents trained on poor trajectories from further improving their policies. The authors show that this issue is particularly relevant when training relies heavily on early data coming from a fixed random policy. We hypothesize that one of the causes for this is also policy confounding. The random policy may induce spurious correlations that lead to the formation of rigid state (state) representations that are hard to recover from. 
\paragraph{Generalization}
Generalization is a hot topic in machine learning. The promise of a model performing well in contexts other than those encountered during training is undoubtedly appealing. In the realm of  reinforcement learning, the majority of research focuses on generalization to 
environments that, despite sharing a similar structure, differ somewhat from the training environment
\citep{kirk2023survey}. These differences range from small variations in the transition dynamics (e.g., sim-to-real transfer; \citealp{higgins2017darla, tobin2017domain, peng2018sim, zhao2020sim}), changes in the observations (i.e., modifying irrelevant information, such as noise: \citealp{mandlekar2017adversarially, ornia2022observational}, or background variables: \citealp{zhang2020invariant, stone2021distracting}), to alterations in the reward function, resulting in different goals or tasks \citep{taylor2009transfer,lazaric2012transfer,muller2021procedural}. Instead, we address the problem of OOT generalization, where the objective is to generalize to different trajectories within the same environment. 
\begin{example} \label{ex:OODvsOOT}
To illustrate the difference between OOD generalization and OOT generalization, let us consider a robot trained via RL to go from our office to the coffee machine, get coffee, and come back, as well as from our office to the printer, make copies, and come back. There are two possible routes to the coffee machine: either through the printer room or through a corridor directly leading to the coffee machine. The path through the printer room is longer, so the robot typically avoids it when coffee is ordered. However, one day, when we order coffee and the corridor is blocked, the robot attempts to go through the printer room and returns with a copy of a new paper titled `Bad Habits' instead of the coffee. This serves as an example of out-of-trajectory generalization. Since the robot is accustomed to obtaining copies in the copy room, it disregards the coffee order. An example of the more general problem of OOD generalization could involve instructing the robot to navigate the office when the floor is wet or to fetch something different, like a glass of water. The crucial distinction is that, in these last two examples, the states the robot visits or the rewards it receives differ. The robot has not been trained on a wet floor, and it has never retrieved a glass of water before. However, in the first example, we would expect the robot to recognize that being in the copy room does not necessarily imply getting copies. To be fair, the blocked corridor represents a change in the environment; nevertheless, this change is intended to prompt the agent to choose an alternative path. It is worth noting that this alternative path was also feasible in the original environment.
\end{example}








\paragraph{State abstraction}
State abstraction is concerned with removing from the representation all that state information that is irrelevant to the task. In contrast, we are worried about learning representations containing too little information, which can lead to state aliasing. Nonetheless, as argued by \citet{McCallum95PhD}, state abstraction and state aliasing are two sides of the same coin. That is why we borrowed the mathematical frameworks of state abstraction to describe the phenomenon of policy confounding. \citet{Li06ISAIM} provide a taxonomy of the types of state abstraction and how they relate to one another. \citet{Givan03AIJ} introduce the concept of bisimulation, which is equivalent to our definition of Markov state representation (Definition \ref{def:markov_state}). \citet{Ferns2006Methods} propose a method for measuring the similarity between two states. \citet{castro2020scalable} notes that this metric is prohibitively expensive and suggests using a relaxed version that computes state similarity relative to a given policy. This is similar to our notion of $\pi$-Markov state representation (Definition \ref{def:pi_markov_state}).  While the end goal of this metric is to group together states that are similar under a given policy, here we argue that this may lead to poor OOT generalization.

\section{Dynamic Bayesian Networks}\label{ap:dbns}
Figures \ref{fig:dbn_key2door} and \ref{fig:dbn_offtrack} show the DBNs for the Key2Door and Diversion environments, respectively.
\begin{figure*}[h!]
\centering
\includegraphics[width=\textwidth]{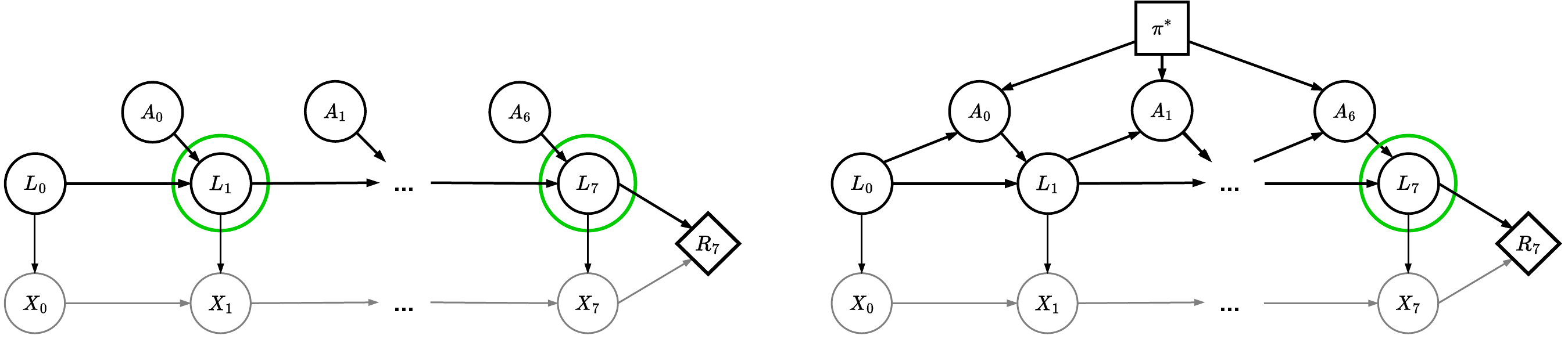}
\caption{Two DBNs representing the dynamics of the Key2Door environment, when actions are sampled at random (left), and when they are determined by the optimal policy (right). The nodes labeled as $L$ represent the agent's location, while the nodes labeled as $X$ represent whether or not the key has been collected. The agent can only see $L$. Hence, when actions that are sampled are random (left), the agent must remember its past locations to determine the reward $R_7$. Note that only $L_1$ and $L_7$ are highlighted in the left DBN. However, other variables in $\langle L_2,\dots, L_6\rangle$ might be needed, depending on when the key is collected. In contrast, when following the optimal policy, only $L_7$ is needed. In this second case, knowing the location is sufficient to determine whether the key has been collected.} 
\label{fig:dbn_key2door}
\end{figure*}
\begin{figure*}[h!]
\centering
\includegraphics[width=\textwidth]{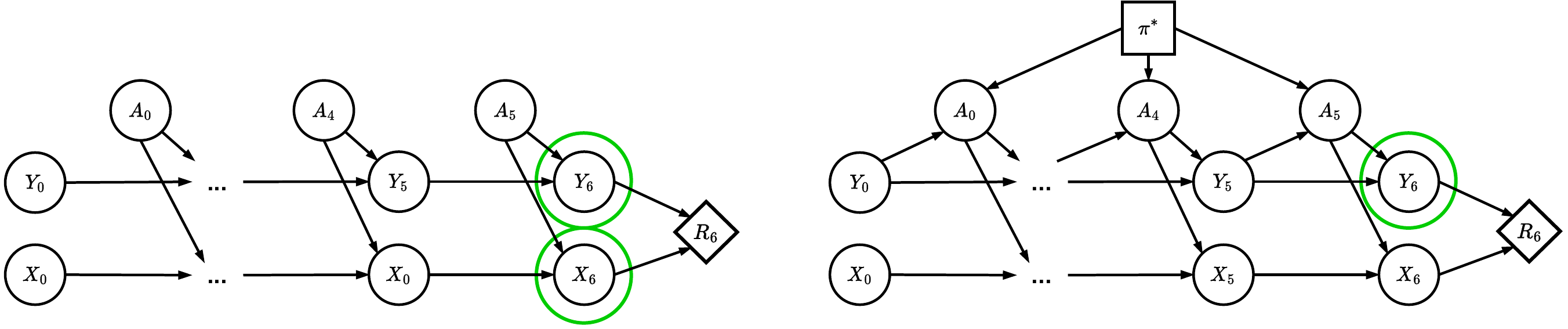}
\caption{Two DBNs representing the dynamics of the Diversion environment, when actions are sampled at random (left), and when they are determined by the optimal policy (right). The nodes labeled as $X$ indicate the row where the agent is located; the nodes labeled as $Y$ indicate the column. We see that when actions are sampled at random, both $X_6$ and $Y_6$ are necessary to determine $R_6$. However, when actions are determined by the optimal policy, $Y_6$ is sufficient, as the agent always stays at the top row.} 
\label{fig:dbn_offtrack}
\end{figure*}

\section{Experimental Results}\label{ap:results}

\subsection{Learned state representations}\label{ap:learned_representations}
The results reported in Section \ref{sec:experiments} show that the OOT generalization problem exists. 
However, some may still wonder if the underlying reason is truly policy confounding.
To confirm this, we compare the outputs of the policy at every state in the Frozen T-Maze when being fed the same states (observation stack) but two different signals. That is, we permute the variable containing the signal ($X$ in the diagram of Figure \ref{fig:tmaze_dbn}) and leave the rest of the variables in the observation stack unchanged. We then feed the two versions to the policy network and measure the KL divergence between the two output probabilities. This metric is a proxy for how much the agent attends to the signal in every state. The heatmaps in Figure \ref{fig:sensitivity} show the KL divergences at various points during training (0, 10K, 30K, and 100K timesteps) when the true signal is `green' and we replace it with `purple'. We omit the two goal states since no actions are taken there. We see that initially (top left heatmap), the signal has very little influence on the policy (note the scale of the colormap is $10^{-6}$), after 10K steps, the agent learns that the signal is very important when at the top right state (top right heatmap). After this, we start seeing how the influence of the signal at the top right state becomes less strong (bottom left heatmap) until it eventually disappears (bottom right heatmap). In contrast, the influence of the signal at the initial state becomes more and more important, indicating that after taking the first action, the agent ignores the signal and only attends to its own location. The results for the alternative case, `purple' signal being replaced by `green' signal, are shown in Figure \ref{fig:sensitivity_1}.
\begin{figure}[ht]
\centering
\includegraphics[width=0.9\textwidth]{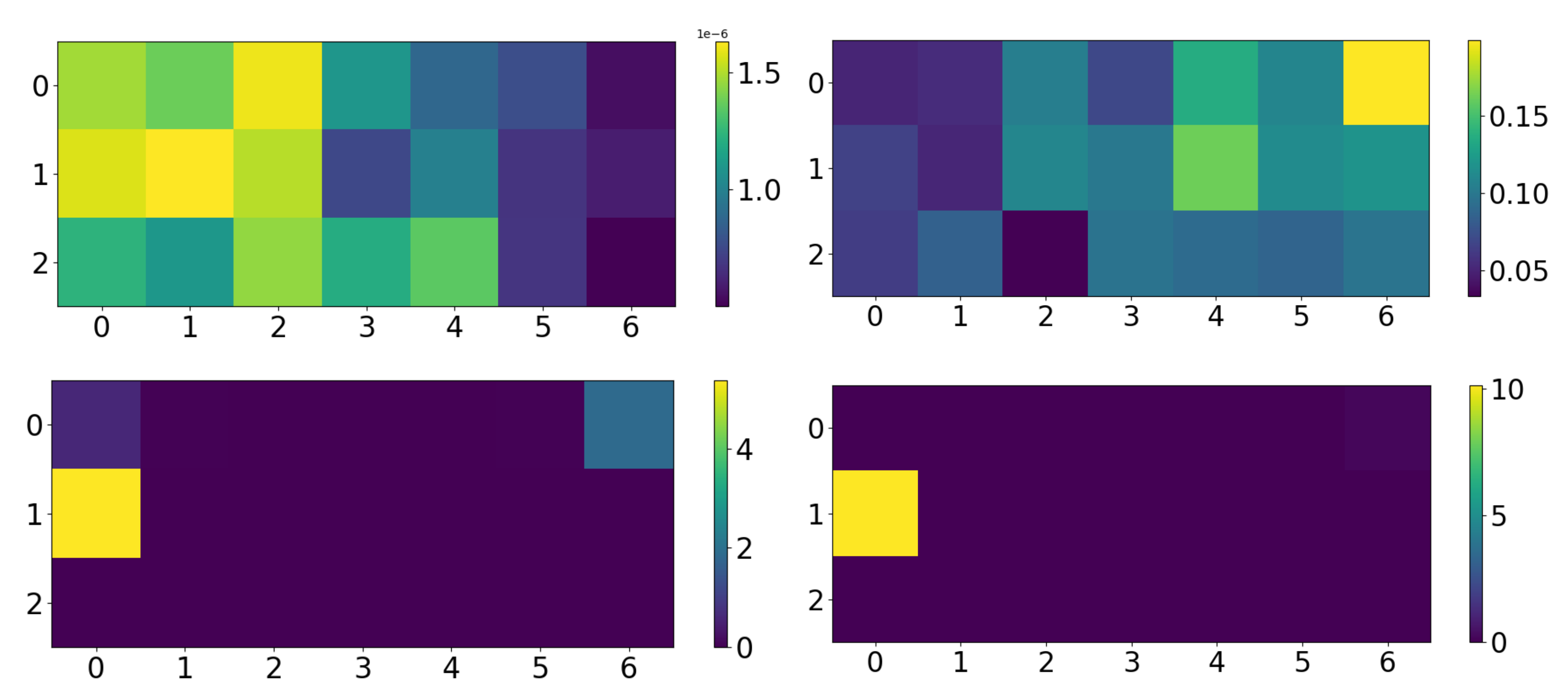}
\caption{A visualization of the learned state representations. The heatmaps show the KL divergence between the action probabilities when feeding the policy network a stack of the past 10 observations and when feeding the same stack but with the value of the signal being switched from green to purple, after  0 (top left), 10K (top right), 30K (bottom left), and 100K (bottom right)  timesteps of training.
} 
\label{fig:sensitivity}
\end{figure}
\begin{figure}[ht]
\centering
\includegraphics[width=0.9\textwidth]{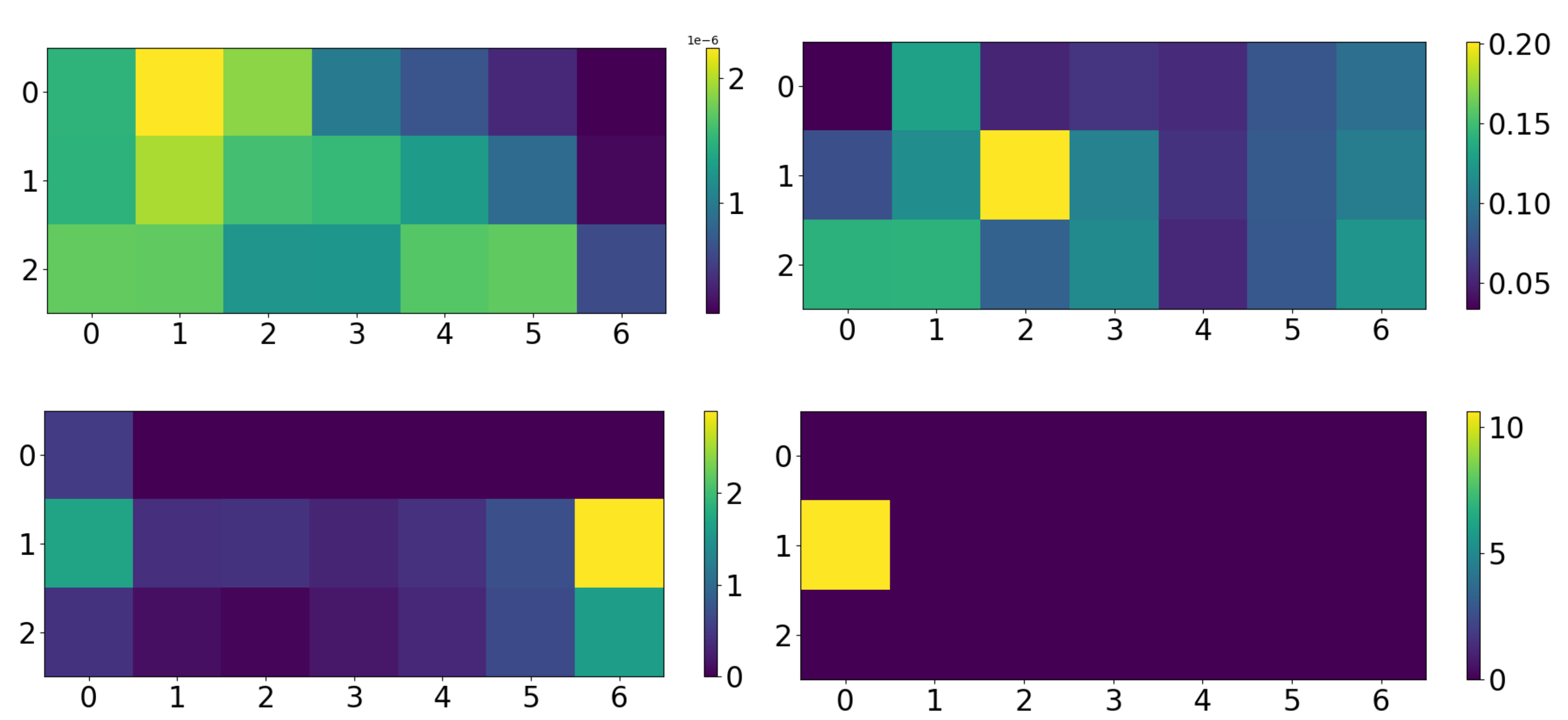}
\caption{A visualization of the learned state representations. The heatmaps show the KL divergence between the action probabilities when feeding the policy network a stack of the past 10 observations and when feeding the same stack but with the value of the signal being switched from purple to green, after  0 (top left), 10K (top right), 30K (bottom left), and 100K (bottom right)  timesteps of training.
} 
\label{fig:sensitivity_1}
\end{figure}

\subsection{Buffer size and exploration/domain randomization}\label{ap:buffer_exploration}
Figures \ref{fig:exploration_buffer_keydoor} and \ref{fig:exploration_buffer_offtrack} report the results of the experiments described in Section \ref{sec:experiments} (paragraphs 2 and 3) for Key2Door and Diversion. We see how the buffer size also affects the performance of DQN in the two environments (left plots). We also see that exploration/domain randomization does improve OOT generalization in Diversion but not in Key2Door.
\begin{figure}[h!]
\centering
\includegraphics[width=0.7\textwidth]{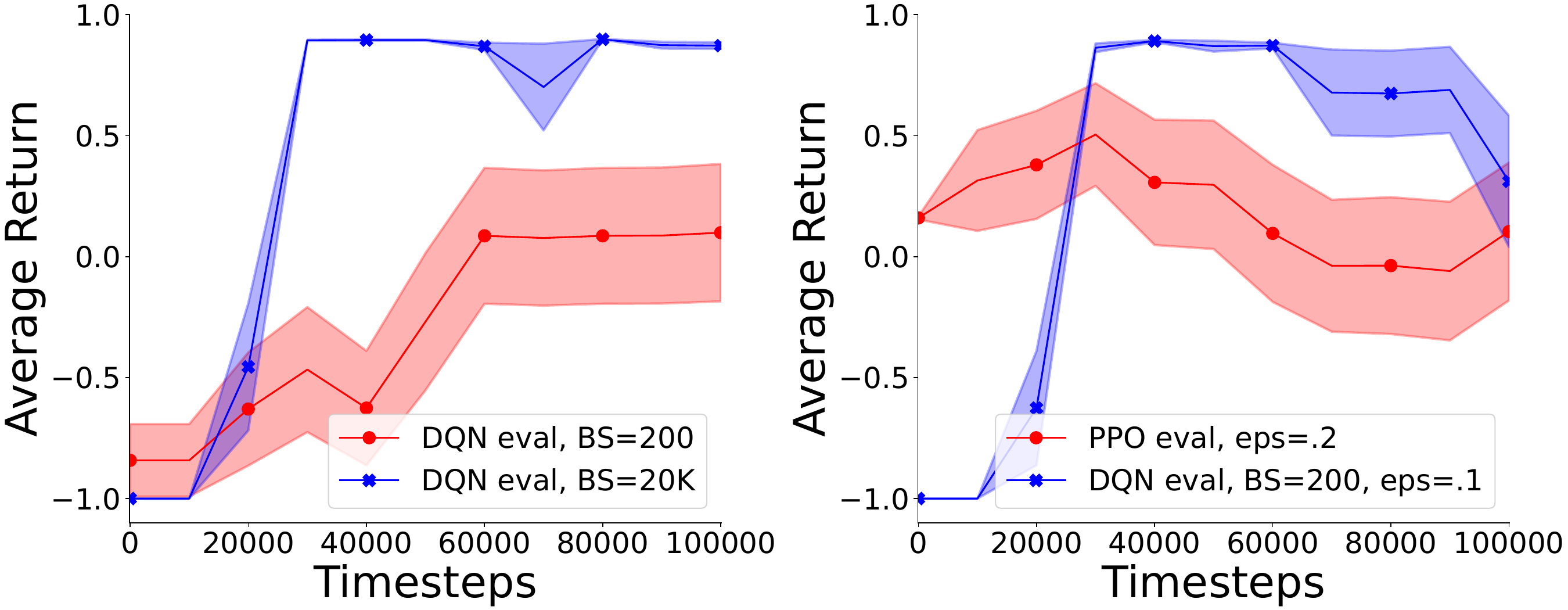}
\caption{Key2Door. Left: DQN small vs. large buffer sizes. Right: PPO and DQN when adding stochasticity.} 
\label{fig:exploration_buffer_keydoor}
\end{figure}

\begin{figure}[h!]
\centering
\includegraphics[width=0.7\textwidth]{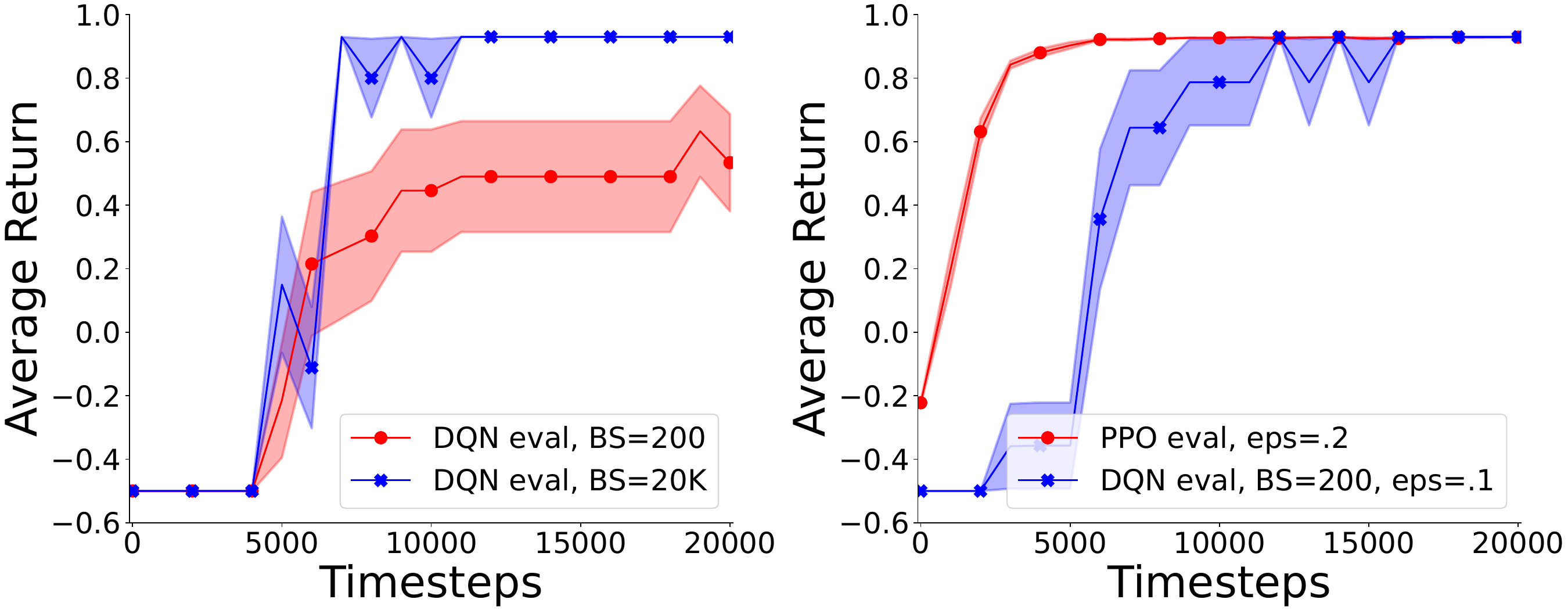}
\caption{Diversion. Left: DQN small vs. large buffer sizes. Right: PPO and DQN when adding stochasticity.}
\label{fig:exploration_buffer_offtrack}
\end{figure}
\newpage
\section{Further Experimental Details}\label{ap:experimental_details}
We ran our experiments on an Intel i7-8650U CPU with 8 cores. Agents were trained with Stable Baselines3 \citep{stable-baselines3}. Most hyperparameters were set to their default values except for the ones reported in Tables \ref{tab:PPO_hyperparameters} (PPO) and \ref{tab:DQN_hyperparameters} (DQN), which worked better than the default values for these particular environments.
\input{hyperparameters}

\end{document}

%% file: hyperparameters.tex
\begin{center}
\vspace{-10pt}
  \begin{table}[h]
  \centering
  \caption{PPO hyperparameters.}
  \vspace{5pt}
\resizebox{0.7\textwidth}{!}{
  \begin{tabular}{ p{5cm}|p{4cm}}
 Rollout steps & 128 \\
 Batch size & 32 \\
 Learning rate & 2.5e-4 \\
 Number epoch & 3 \\
 Entropy coefficient & 1.0e-2 \\
 Clip  range & 0.1 \\
 Value coefficient & 1 \\
 Number Neurons 1st layer & 128 \\
 Number Neurons 2nd layer & 128 \\
\end{tabular}
}
\vspace{-20pt}
\label{tab:PPO_hyperparameters}
\end{table}
\end{center}

\begin{center}
\vspace{-10pt}
  \begin{table}[h]
  \centering
  \caption{DQN hyperparameters.}
  \vspace{5pt}
\resizebox{0.7\textwidth}{!}{
  \begin{tabular}{ p{5cm}|p{4cm}}
 Buffer size & 1.0e5 \\
 Learning starts & 1.0e3 \\
 Learning rate & 2.5e-4 \\
 Batch size & 256 \\
 Initial exploration bonus  & 1.0 \\
 Final exploration bonus & 0.0 \\
 Exploration fraction & 0.2 \\
 Training frequency & 5 \\
 Number Neurons 1st layer & 128\\
 Number Neurons 2nd layer & 128\\
\end{tabular}
}
\vspace{-20pt}
\label{tab:DQN_hyperparameters}
\end{table}
\end{center}